\documentclass[12pt,a4paper]{article}
\usepackage{amsmath,amssymb,amsthm}
\usepackage{epsfig}
\usepackage{graphics}
\usepackage{graphicx}

\title{Shape Characterization via Boundary Distortion}

\author{Xavier Descombes,\\
{\it INRIA SAM Sophia Antipolis, FR} \and
Serguei Komech \thanks{The work
is partially supported by RFBR grant 12-01-31294}\\
{\it Dobrushin Lab. IITP, Moscow, RU}\\
komech@iitp.ru
}

\newtheorem{definition}{Definition}
\newtheorem{proposition}{Proposition}

\newcommand{\be}{\begin{equation}}
\newcommand{\ee}{\end{equation}}

\newcommand{\ptb}{P^{\varepsilon}_{(\theta,\beta)}}

\begin{document}

\maketitle

\begin{abstract}
In this paper, we derive new shape descriptors based on a directional characterization. The main idea is to study the behavior of the shape neighborhood under
family of transformations.
We obtain a description invariant with respect to rotation, reflection, translation and scaling.
A well-defined metric is then proposed on the associated feature space. We show the continuity of this metric.
Some results on shape retrieval are provided on two databases to show the accuracy of the proposed shape metric.
\end{abstract}

\section{Introduction}
Shape characterization is becoming a crucial challenge in image analysis. The increasing resolution of new sensors, satellite images
or scanners provides information on the object geometry which can be interpreted by shape analysis.
 The size of data basis also requires some efficient tools  for analyzing shapes, for example in applications such as image retrieval.
 This task is not straightforward. If the goal of shape analyzing is to recognize a 3D object, for instance for classification
 or image retrieval purposes, then the data only consist of a 2D projection of the object. Therefore,
 one dimension is ``lost''. Besides, some  noise may affects the object boundary or more precisely the silhouette of the object
 in the considered image. To address this problem, numerous techniques and models have been proposed.
 Reviews of proposed representations can be found in~\cite{Loncaric98,Zhang04}. One class of methods consists in defining
 shapes descriptors based on shape signatures histogram signatures, shape invariant moments, contrast,
  matrices or spectral features. A shape representation is evaluated with respect to its robustness,
  w.r.t. noise and/or intra-class variability, compacity of the description, its invariance properties and its efficiency
  in terms of computation time. According to Zhang and Lu~\cite{Zhang04}, the different approaches can be classified
  into contour-based and region-based methods, and within each class between structural and global approaches.
  In this paper we consider shapes as binary silhouette of objects and concentrate on global approaches.
  We propose some feature vectors and define a metric in the feature space.  Following~\cite{Zhang04}, we can
  distinguish several global approaches. Simple global shape descriptors embed area, orientation, convexity, bending
  energy~\cite{Yong74,Peura97}. Usually, these descriptors are not sufficiently sensitive to details to provide good scores
  in image retrieval. Distances between shapes or surfaces have been proposed, such as the Hausdorf distance or some modification
  to reduce sensitivity to outlier~\cite{Rucklidge97,Belongie01}. In this setting, the invariance properties can be obtained by
  taking the minimum distance over the corresponding group of transformation. A key issue is to consider a metric for which
  the minimum is computed with a low computational complexity. Shape signatures based on the boundary give a 1D function as
  for example the angle function~\cite{Srivastava}, the curvature or the chord-length~\cite{Wang06}. Using these signatures,
  a slight change in the contour may result in a big change in the signature. Therefore, special care is required for defining
  a metric on the signature space. To reduce the dimension of the representation, boundary or surface moments can be used.
  They usually embed good invariance properties and are fast to compute. The geometric moments introduced in~\cite{Hu62}
  and extended, for example for 3D objects in~\cite{Xu08}, are limited in the complexity of shapes they can handle.
  Usually, lower order moments do not reflect enough information and higher order moments are difficult to estimate.
  Preferable alternatives are the algebric moments~\cite{Kaveti97} or the fourier descriptors~\cite{Zhang02}.
  Stochastic models of the shape or the coutour have also been proposed. For example, autoregressive models of the boundary
  provide some shape descriptors~\cite{Dubois86,He91}. However, the problem of choosing the order of the model is still open.
  Considering too many parameters leads to an estimation issue. moreover, the interpretation of parameters in terms of shape
  properties is not clear. Studying the shape at different scales as motivated different work. The shape is then described by
  its inflection points after a Gaussian filtering~\cite{Kopf05}. A distance can also be derived by matching scale space
  images~\cite{Daoudi00}. Finally, the analysis can be performed using spectral transforms, such as Fourier~\cite{Kunttu05,Capar06}
  or wavelet~\cite{Li04,Kong07} descriptors. The issues are then to set the number of relevant coefficients and the definition of a
  metric between these features.

In this paper, we derive a 2D signature of shapes and propose a metric on the associated feature space. The first idea consists
of a description of the boundary regularity by comparing the volume of the boundary neighborhood with the shape volume.
The second idea is to study the behavior of this descriptor under shape transformations. We thus define a family of
diffeomorphisms consisting in expanding the shape in one direction and contracting it in the orthogonal direction.
In that way, for a given detail, there exists at least such a transformation enlarging its contribution to the descriptor
and another one reducing it. We then derive a well defined metric on the feature space, and show its performance for shape
discrimination on databases of various size.

The paper is organized as follows. We describe the proposed shape space and define a metric on it in section~\ref{sec:topo}.
A discretization of the metric is described and evaluated on two different databases in section~\ref{sec:result}.
Finally, conclusion and perspectives are drawn in section~\ref{sec:conclusion}.

\section{A topological description of shapes}
\label{sec:topo}
We consider shapes as 2D silhouettes of bounded objects in the image plane:

\subsection{Shape space }
\begin{definition}
The pre-shape space $S$ is the set of subsets of
$\mathbb{R}^2$ satisfying the following conditions:
\begin{itemize}
\item[C1:] $\forall a \in S,\ a$ is compact and connected, with a strictly positive area,
\item[C2:] $\forall a \in S,\ \mathbb{R}^2 \setminus a$ is connected ($a$ has no hole).
\end{itemize}
\end{definition}

Let us consider a shape $a\in S$. Define the closed $\varepsilon$-neighbourhood
of the set $a$ in the sense of the Euclidean metric as
$O^{\varepsilon} (a) = \{x\in \mathbb{R}^2 : e(x,a)\leq \varepsilon \}, \varepsilon\geq 0$, where
$e(\cdot,\cdot)$ is the Euclidean distance.

On this pre-shape space, we consider the Hausdorff metric (which is well-defined, see, for example, \cite{Serra}) for the sets in $\mathbb{R}^2$:
$$
\rho(a,b)=\inf\{\delta>0 : a\subset O^{\delta}(b), b\subset O^{\delta}(a) \},
$$
where $a,b \subset S$.

A shape space should embed  some invariance properties. Let $G$ be the group of transformations of $\mathbb{R}^2$ generated
by rotations, translations, reflections and scaling : $G = SO^{\pm}_2(\mathbb{R})\times \mathbb{R}_{+}$. To define a shape space $\mathbb{S}$
isometry- and scale-invariant, we consider:
\begin{equation}
\mathbb{S}= S / G.
\end{equation}
For a given $A \in \mathbb{S}$, we note $r(A) = \{ a  \in S : vol(a)=1, G(a) = A \}$, where vol($\cdot$) is the area of the set.

Therefore, on the shape space $\mathbb{S}$, the Hausdorff metric becomes:

\begin{equation}
d(A, B)=\inf \{\rho(a,b)\ |\ a\in r(A), b\in r(B)\}
\end{equation}
where $A,B\in \mathbb{S}$ (note that this metric can be compared with the Procrustes distance for sets consisting of finite number of points~\cite{Ke,DM}).

\begin{proposition}
$d(\cdot,\cdot)$ is a well-defined metric on $\mathbb{S}$.
\end{proposition}

\begin{proof}

Let us consider $A,B,C \in \mathbb{S}$.
It is straightforward that $d(A,B)=d(B,A)$ and $d(A,B)=0 \Leftrightarrow r(A)=r(B)$ due to the compactness of the considered sets.
Then, we only have to check the following property: $d(A,C)\leq d(A,B)+d(B,C)$.
Suppose that there exist $A,B,C$ such that
\begin{equation}
d(A,C)>d(A,B)+d(B,C).
\end{equation}
Let $\delta:=d(A,C)-d(A,B)-d(B,C)>0$. By definition there exist $a\in~r(A),\\ b_1,b_2\in r(B),\ c\in r(C)$
such that
\begin{equation}\label{contra}
 \begin{matrix}
(p_1)~ \rho(a,b_1)<d(A,B)+\delta/4,\\
(p_2)~ \rho(b_2,c)<d(B,C)+\delta/4.
 \end{matrix}
\end{equation}
Then $\exists g \in G : b_1 = g(b_2)$ so that $(p_3) \rho(b_2,c) = \rho(b_1,g(c))$ and $g(c) \in r(C)$.
Therefore, we have $c_1 = g(c) \in r(C),\ \rho(b_2,c)=\rho(b_1,c_1)$.

We have:

\begin{equation}
 (p_1) \Rightarrow  b_1 \subset O^{d(A,B) + \delta/4}(a),~~ (p_2+p_3) \Rightarrow  c_1 \subset O^{d(B,C) + \delta/4}(b_1),
\end{equation}

and:

\begin{equation}
 (p_1) \Rightarrow  a \subset O^{d(A,B) + \delta/4}(b_1),~~ (p_2+p_3) \Rightarrow  b_1 \subset O^{d(B,C) + \delta/4}(c_1).
\end{equation}
Therefore:

\begin{equation}
\begin{matrix}
c_1\subset O^{d(B,C)+\delta/4+d(A,B)+\delta/4}(a),\\
a\subset O^{d(A,B)+\delta/4+d(B,C)+\delta/4}(c_1).
\end{matrix}
\end{equation}
Hence,
\be
d(A,C)\leq \rho(a,c_1)\leq d(A,B)+d(B,C)+\delta/2.
\ee
This contradiction ends the proof.
\end{proof}

\subsection{ Volume descriptor and family of transformations}

The main idea of the proposed description is to characterize the behavior of shapes under some transformations. These transformations aim at
enlighting small characteristic details. We first consider the volume behavior under some dilation. Intuitively, this volume will increase more
for sinuous shape boundaries than for smooth shapes.

Let us consider a shape $a\in S$. Idea of our shape descriptor is based on analyzing the fraction
\begin{equation}
\label{n}
P^{\varepsilon}(a) = \frac {vol(O^{\varepsilon}(a) \setminus a)} {vol(a)} \quad ,
\end{equation}
where vol($\cdot$) is the area of the set (for $a\subset\mathbb{Z}^2$,
it would be the number of pixels). This parameter is well-defined
as we only consider nonzero area set $a$.

Basically, the proposed feature study the evolution of the ratio between the neighborhood volume and the volume of the shape after some dilation. It provides some information on the smoothness of the contour and on the size of the contour concavities.
To complete the description, the next step consists in enlarging details in shapes for a more robust discrimination.
Besides, we consider a directional analysis by defining family of transformations parametrized by an angle and coefficient of expansion.

We consider a family $\{F_{(\theta, \beta)}\}$ of linear transformations of $\mathbb{R}^2$ in order to obtain more significant information about shapes.
The goal of such transformations is to emphasize "features" of the shape in specific direction.
These transformations are defined as follows:

\begin{equation}
F_{(\theta, \beta)} : \begin{pmatrix}
\beta \cos^2\theta +  \frac{1}{\beta}\sin^2\theta & (\beta-\frac{1}{\beta})\sin\theta\cos\theta  \\
(\beta-\frac{1}{\beta})\sin\theta\cos\theta & \beta \sin^2\theta +
\frac{1}{\beta}\cos^2\theta
\end{pmatrix},
\end{equation}
where $\theta\in [-\frac{\pi}{2}, \frac{\pi}{2}], \beta\geq1$. Every $F_{(\theta,\beta)}$ is $\beta$-times expanding in one direction
and $\beta$-times contracting in orthogonal, so it is a volume-preserving transformation.\\

For every set $a\subset \mathbb{R}^2$, we obtain the map
\begin{equation}
a \rightarrow P^{\varepsilon}_{(\theta,\beta)}(a):=\frac{vol(O^{\varepsilon}(F_{(\theta,\beta)}a)\setminus F_{(\theta,\beta)}a)}{vol(a)},\quad \theta\in [-\frac{\pi}{2}, \frac{\pi}{2}],\ \beta\geq 1,\ \varepsilon>0.
\end{equation}
It is clear that $P^{\varepsilon}_{(\theta,\beta)}(a)$ is a continuous function of $\varepsilon,\theta$ and $\beta$
and $P^{\varepsilon}_{(-\frac{\pi}{2},\beta)}(a)=P^{\varepsilon}_{(\frac{\pi}{2},\beta)}(a)$.

On figure~\ref{fig:hand01} is shown a hand shape $a$ for which $P^n_{(\theta,1)}(a) = 0.23$. After transformations, we obtain respectively  $P^n_{(0,2)}(a) = 0.33$
and $P^n_{(90,2)}(a) = 0.20$. Therefore, when expanding the shape in the finger direction, the descriptor increases while slightly decreasing when contracting along this direction.

\begin{figure}
  \center
  \includegraphics[height=3cm]{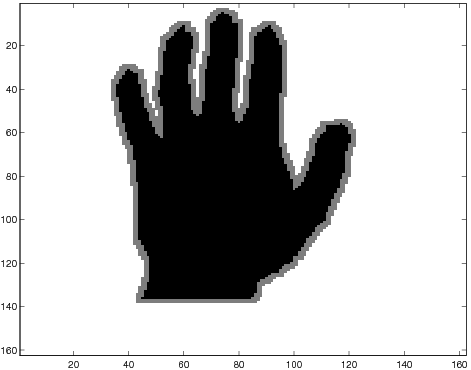}
  \includegraphics[height=3cm]{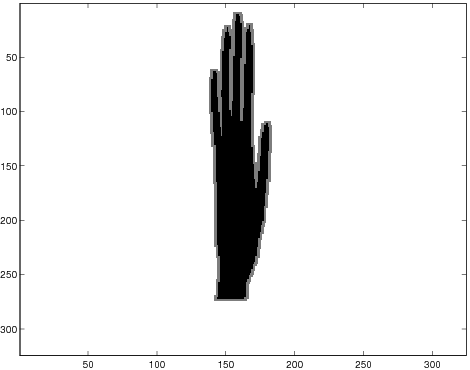}
 \includegraphics[height=3cm]{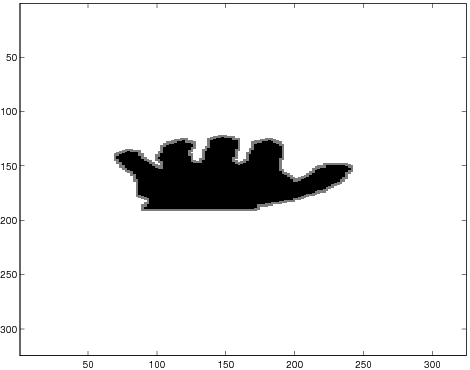}
  \caption{Hand (left), after transformation $F_{(\pi/2,2)}$ (middle), and after transformation $F_{(0,2)}$}\label{fig:hand01}
\end{figure}

Consider $R_{\gamma}$, the rotation by an angle $\gamma$. We have the following property:

\begin{proposition}
\label{Rot}
$\forall \theta \in [-\frac{\pi}{2},\frac{\pi}{2}], \ptb(R_{\gamma}a) = P^{\varepsilon}_{((\theta+\frac{\pi}{2}+\gamma)mod(\pi) - \frac{\pi}{2}, \beta)}(a)$
\end{proposition}

Consider $R_x$, the reflection with respect to $x$ axis (horizontal line).  We have the following property:

\begin{proposition}
\label{Ref}
$\forall \theta \in [-\frac{\pi}{2},\frac{\pi}{2}], \ptb(R_x a) = P^{\varepsilon}_{(-\theta, \beta)}(a)$
\end{proposition}

Let us denote $R$ the group of transformations generated by properties~\ref{Rot} and \ref{Ref}. The shape representation space $\mathcal{R}$ we consider for the fixed $\varepsilon>0$ is then defined by the following mapping:

\begin{eqnarray}
\Phi : \mathbb{S} & \rightarrow & \mathcal{R} = C^0([-\frac{\pi}{2}, \frac{\pi}{2}]\times [1,\infty])/R \nonumber \\
A & \mapsto & \ptb(a)/R, \quad a\in r(A).
\end{eqnarray}

On figure~\ref{fig:calf2}, we can remark that the function $P^n_{(\gamma,\beta)}(a)$ increases more in two directions, corresponding to an extention of both the calf body and its legs.
On figure~\ref{fig:birdfight1}, the biggest slope is obtained when expanding the wings of the birdfight.

\begin{figure}
  \center
  \includegraphics[height=4.5cm]{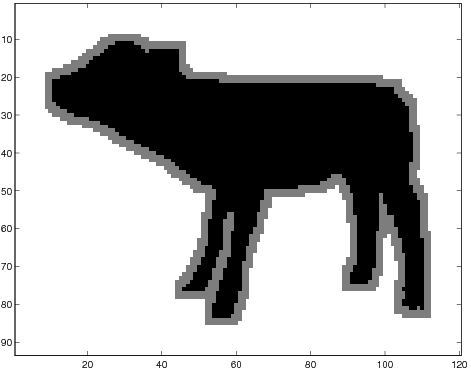}
  \includegraphics[height=4.5cm]{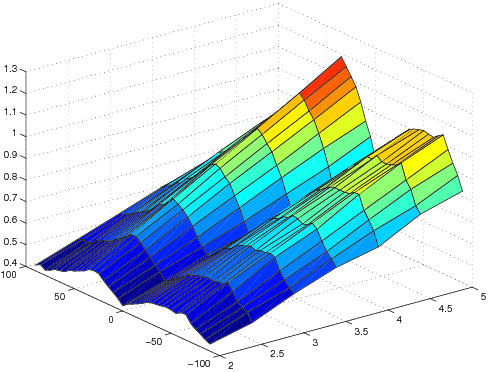} \\
 \includegraphics[height=2.2cm]{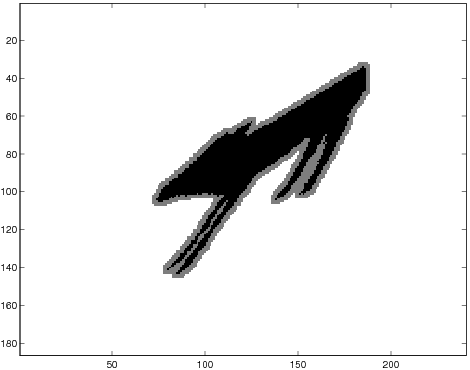}
 \includegraphics[height=2.2cm]{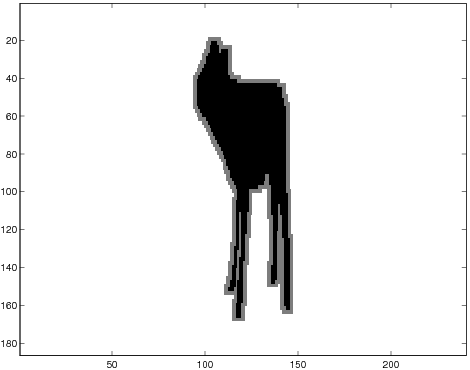}
 \includegraphics[height=2.2cm]{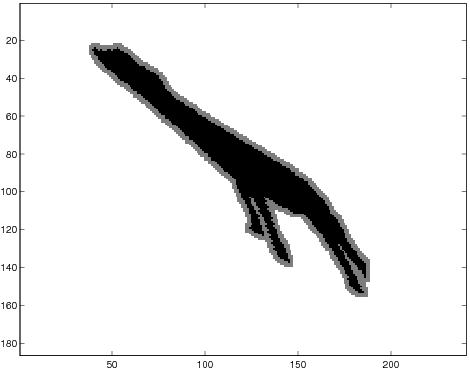}
 \includegraphics[height=2.2cm]{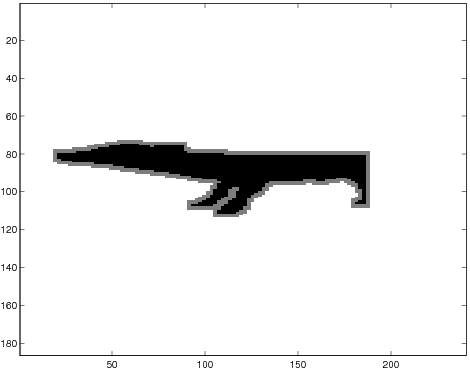}
  \caption{Calf (top left) and the associated representation $P_n(F_{(\gamma,\beta)})$.
The bottom line represents $F_{(\gamma,\beta)}$ for $\beta=2$ and $\gamma = \pi/4,\pi/2,-\pi/4,0$ (the shape is in black and the neighborhood in grey).}\label{fig:calf2}
\end{figure}

\begin{figure}
  \center
  \includegraphics[height=4.5cm]{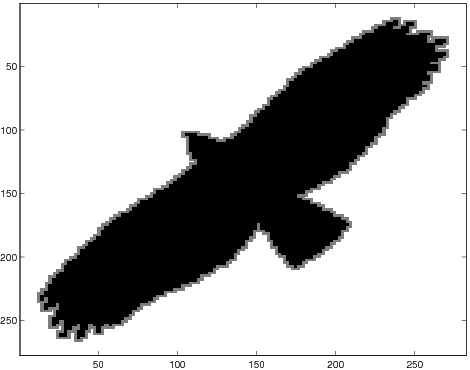}
  \includegraphics[height=4.5cm]{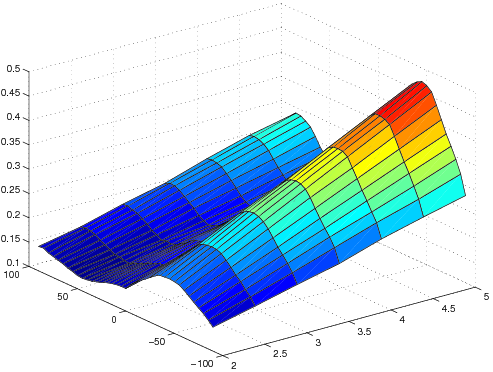} \\
 \includegraphics[height=2.2cm]{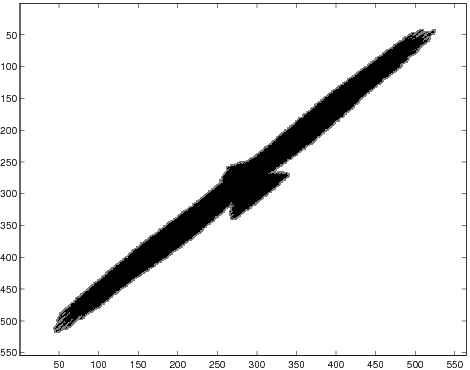}
 \includegraphics[height=2.2cm]{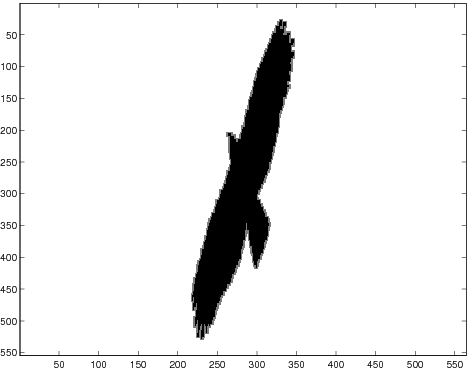}
 \includegraphics[height=2.2cm]{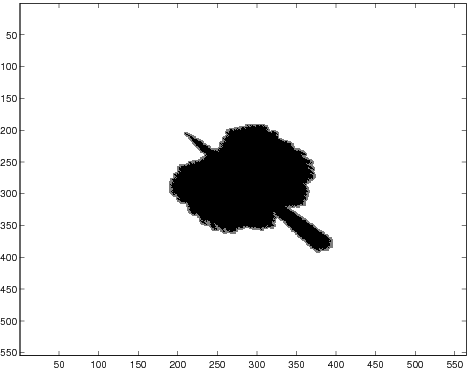}
 \includegraphics[height=2.2cm]{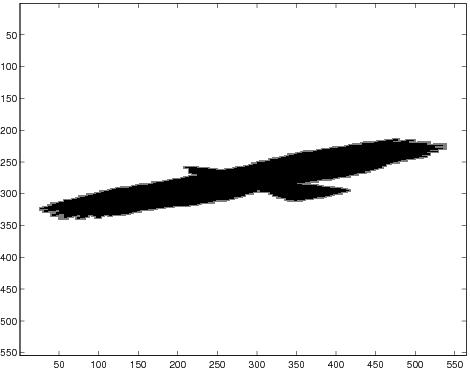}
  \caption{Birdfight (top left) and the associated representation $P_n(F_{(\gamma,\beta)})$.
The bottom line represents $F_{(\gamma,\beta)}$ for $\beta=2$ and $\gamma = \pi/4,\pi/2,-\pi/4,0$ (the shape is in black and the neighborhood in grey).}\label{fig:birdfight1}
\end{figure}

We then consider the following metric on $\mathcal{R}$:

\begin{equation}
\label{eq:themetric}
l(\Phi(A), \Phi(B)):= \inf_{a,b}\left(\int\limits_{[-\frac{\pi}{2}, \frac{\pi}{2}]\times [1,\infty]} (\ptb(a) -\ptb(b))^2 \ e^{-\kappa\beta} d\beta \ d\theta \right)^{1/2},
\end{equation}
where $ \kappa>0,~ a\in r(A),~ b\in r(B)$. The integral on the right-hand side converges due to $\ptb(a)$ is almost linear function of $\beta$ for $\beta$ big enough.

We thus have defined a map between the shape space and the feature space. Two similar shapes should be associated to close points in the feature space.
This property can be established by the continuity of mapping $\Phi$ with respect to the metrics defined in both spaces.

\begin{proposition}
$\Phi: (\mathbb{S},d(,))\to (\mathcal{R},l(,))$ is a continuous map.
\end{proposition}

\begin{proof}

We consider shape $A\in\mathbb{S}$. For the $\alpha>0$ we would like to find $\delta>0$ such that
\begin{equation}
d(A,B) < \delta \Rightarrow l(\Phi(A), \Phi(B))<\alpha.
\end{equation}

Note, that $\forall a\in r(A)$, $\ptb(a)=vol(O^{\varepsilon}(F_{(\theta,\beta)}a))-1$.

First we choose $\beta_0$ such that $\forall a\in r(A), b\in r(B)$
$$
\int\limits_{[-\frac{\pi}{2}, \frac{\pi}{2}]\times [\beta_0,\infty]}(\ptb(a) -\ptb(b))^2 \ e^{-\kappa\beta} d\beta \ d\theta \leq
$$
\begin{equation}
\leq \int\limits_{[-\frac{\pi}{2}, \frac{\pi}{2}]\times [\beta_0,\infty]} (C \beta \varepsilon)^2 e^{-\kappa\beta} d\beta \ d\theta< \alpha^2/2,
\end{equation}
where $C$ depends only on diameter of $a$.

By definition there exists $a,b$:
\begin{equation}
a \subset O^{\delta}(b),~b \subset O^{\delta}(a),
\label{eq:defdist}
\end{equation}
where $a\in r(A),b\in r (B)$. 
Equation~(\ref{eq:defdist}) implies that $\forall \theta,\beta$:

\begin{equation}\label{fb}
F_{(\theta,\beta)}(a) \subset O^{\beta \delta}\left( F_{(\theta,\beta)}(b)\right),~~
F_{(\theta,\beta)}(b) \subset O^{\beta \delta}\left( F_{(\theta,\beta)}(a)\right).
\end{equation}

Using (\ref{fb}) we choose $\delta(\theta,\beta)>0$ such that $\forall \delta<\delta(\theta,\beta),\ \beta<\beta_0,$
\begin{eqnarray}
\left|\ptb(b) - \ptb(a) \right|^2 & =& \left| vol(O^{\varepsilon}(F_{(\theta,\beta)}b)) -  vol(O^{\varepsilon}(F_{(\theta,\beta)}a))  \right|^2 \nonumber\\
 & \leq & \left( vol(O^{\varepsilon+\beta\delta}(F_{(\theta,\beta)}a)) - vol(O^{\varepsilon}(F_{(\theta,\beta)}a))\right)^2 \nonumber\\
 & \leq  &  \alpha^2/2w,
\end{eqnarray}
where $w=\int_{[-\frac{\pi}{2}, \frac{\pi}{2}]\times [1,\beta_0]} e^{-\kappa\beta} d\beta \ d\theta$.
It is easy to verify that $\delta(\theta,\beta)$ is a continuous function over compact set $[-\frac{\pi}{2}, \frac{\pi}{2}]\times [1,\beta_0]$.
Let $\min \delta(\theta,\beta)=\delta_0>0$ on this set.

  Finally, for $\delta<\delta_0$:

$$
l(\Phi(A),\Phi(B))^2 \leq \int (\ptb(a) -\ptb(b))^2 \ e^{-\kappa\beta} d\beta \ d\theta \leq
$$

\begin{equation}
\leq  \int_{[-\frac{\pi}{2}, \frac{\pi}{2}]\times [1,\beta_0]} (\ldots) e^{-\kappa\beta} d\beta \ d\theta
 +  \int_{[-\frac{\pi}{2}, \frac{\pi}{2}]\times [\beta_0,\infty]} (\ldots) e^{-\kappa\beta} d\beta \ d\theta\\
\leq  w \alpha^2/2w  +  \alpha^2/2.
\end{equation}

And so
\begin{equation}
 l(\Phi(A),\Phi(B))<\alpha.
\end{equation}
\end{proof}

\section{Implementation and results}
\label{sec:result}

\subsection{Discretization}
In practice, to compute the distance between two shapes, we have to discretize equation~\ref{eq:themetric}.
When analysing the surfaces representing
the function $P^n(F_{(\theta,\beta)})$ on figures~\ref{fig:calf2} and~\ref{fig:birdfight1}, it is clear that the embeded information is redundant.
Indeed, the surfaces are very smooth, so that we can employ a drastic discretization scheme, without loosing information.

Before computing the proposed feature, we normalize the shapes to V-area in order to satisfy the scale invariance (notice that our descriptor
is invariant with respect to isometries). Indeed,
for $\{a,b\} \in S \times S$ of the same shape but with different volume, we
have to choose different $n_a$ and $n_b$ (see (\ref{n})) in order to obtain $P^{n_a}(a)=
P^{n_b}(b)$. That is why we "normalize"  every set to some fixed area $V>0$ by the corresponding
homothety.

We consider four directions, $\theta \in \{-\frac{\pi}{4}, 0 , \frac{\pi}{4}, \pi \}$, and two expanding coefficients $\beta \in \{ 3, 5 \}$.


\subsection{MPEG-7 CE Shape-1 Part-B data set}
We first evaluate the proposed approach of the MPEG-7 CE Shape-1 Part-B data set (see \cite{kimia}), composed of 7 classes, containing each 20 shapes.
Although the classes are quite distinct, this data set contains important within-class variations (see figure~\ref{fig:shape4}).

We consider four directions, $\theta \in \{-\frac{\pi}{4}, 0 , \frac{\pi}{4}, \pi \}$, and two expanding coefficients $\beta \in \{ 3, 5 \}$.
The coefficient defined the metric is $\kappa = \frac{1}{5}$.

\begin{figure}
  \center
  \includegraphics[height=5cm]{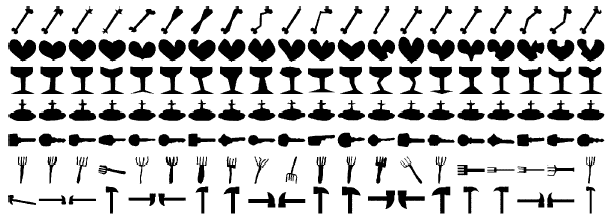}
  \caption{The MPEG-7 CE Shape-1 Part-B data set}\label{fig:shape4}
\end{figure}

We consider the proposed metric between each pair of shapes (except itself of course, cause distance is $0$) and report in Table~\ref{tab:shape4} the percentage of correct $n^{th}$ neighbors for
each class. The total correct answers correspond to $98\%$ for the first neighbors and $95\%$ for the second neighbors. If we consider the tenth neighbors,
we still obtain a total score of $85\%$ of good retrieval. This shows the robustness of the proposed metric.
\begin{table}
\begin{center}
\begin{tabular}{|c|c|c|c|c|c|c|c|c|c|c|}
\hline
& 1st & 2nd & 3rd  & 4th & 5th & 6th & 7th & 8th & 9th & 10th \\
\hline
Bonefull & 95 & 70 & 85 & 85 & 90 & 85 & 85 & 85 & 70 & 75\\
\hline
Heart & 100 & 100 & 100 & 100 & 100 & 100 & 95 & 95 & 95 & 100 \\
\hline
Glas & 100 & 100 & 100 & 100 & 100 & 100 & 100 & 90 & 100 & 95 \\
\hline
Fountain & 100 & 100 & 100 & 100 & 100 & 100 & 100 & 100 & 100 & 100\\
\hline
Key & 100 & 100 & 95  & 100 & 95 & 95 & 90 & 90 & 95 & 95 \\
\hline
Fork & 95 & 90 & 65 & 70 & 65 & 75 & 65 & 75 & 70 & 60 \\
\hline
Hammerfull & 95 & 95 & 80 & 30 & 30 & 35 & 30 & 40 & 35 & 15 \\
\hline
\end{tabular}
\caption{Retrieval scores on the MPEG-7 database} \label{tab:shape4}
\end{center}
\end{table}

\subsection{Kimia database}
We now consider a database defined by Kimia. It consists in 676 shapes divided into 27 classes (see additional material). The global retrieval
score is $94 \%$ for the first neighbor, $75 \%$ for the second neighbor, $69 \%$ for the third neighbor, $70\%$ for the fourth neighbor
and $67\%$ for the fifth neighbor. The results obtained for each class are summarized in table~\ref{tab:kimia}. They show the robustness
of the proposed metric in case of a huge database. However, notice that we do not model the objects themselves but only consider
shape descriptors without semantic interpretation. Therefore, the proposed metric is not adapted to occluded shapes.
On figure~\ref{fig:occluded}, the three first neighbors of a hand, occluded by a bar, are elephants. This result is natural considering
the number and the size of growths and their angle distribution.

\begin{table}
\begin{center}
\begin{tabular}{|c|c|c|c|c|c|c|c|c|c|c|c|c|c|c|}
\hline
Class & 1 & 2 & 3 & 4 & 5 & 6 & 7 & 8 & 9 & 10 & 11 & 12 & 13 & 14 \\
\hline
$1^{st}$ & 88 & 85 & 91 & 90 & 91 & 100 & 100 & 95 & 83 & 83 & 71 & 70 & 96 & 75\\
\hline
$2^{nd}$. & 53 & 70 & 88 & 90 & 100 & 100 & 100 & 100 & 83 & 71 & 63 & 63 & 87 & 47 \\
\hline
$3^{st}$ & 29 & 75 & 91 & 90 & 75 & 100 & 100 & 100 & 50 & 46 & 55 & 44 & 96 & 32 \\
\hline
Class & 15 & 16 & 17 & 18 & 19 & 20 & 21 & 22 & 23 & 24 & 25 & 26 & 27 &  \\
\hline
$1^{st}$ & 71 & 84 & 70 & 95 & 95 & 85 & 76 & 90 & 80 & 70 & 77 & 100 & 80 &  \\
\hline
$2^{nd}$ & 59 & 72 & 59 & 85 & 95 & 95 & 66 & 95 & 45 & 40 & 69 & 100 & 78 &  \\
\hline
$3^{st}$ & 49 & 75 & 49 & 55 & 90 & 90 & 69 & 95 & 15 & 50 & 72 & 100 & 83 & \\
\hline
\end{tabular}
\caption{The Kimia database classes and the percentage of good neighbor retrieval} \label{tab:kimia}
\end{center}
\end{table}

\begin{figure}
  \center
 \includegraphics[height=2.2cm]{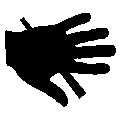}
 \includegraphics[height=2.2cm]{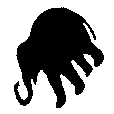}
 \includegraphics[height=2.2cm]{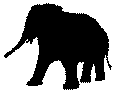}
 \includegraphics[height=2.2cm]{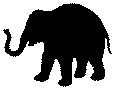}
  \caption{A hand shape occluded with a bar and its three first neighbors}\label{fig:occluded}
\end{figure}

\section{Conclusion}
\label{sec:conclusion}
We have proposed a new metric on a shape space based on the shape properties after applying family of transformations.
The proposed metric is well-defined and continuous. Retrieval results on two databases, one of them consisting of 676 shapes,
divided in 27 classes, have proven the relevance of this metric. We are currently studying the injectivity of the associated mapping. Further studies also include the definition of a shape classification algorithm,
based on this description.


\end{document}